\documentclass{article}
\usepackage{ijcai26}
\usepackage{macro}

\usepackage{times}
\usepackage{soul}
\usepackage{url}
\usepackage[hidelinks]{hyperref}
\usepackage[utf8]{inputenc}
\usepackage[small]{caption}
\usepackage{graphicx}
\usepackage{amsmath}
\usepackage{amsthm}
\usepackage{booktabs}
\usepackage{algorithm}
\usepackage{algorithmic}
\usepackage{amssymb}
\usepackage{subcaption}
\usepackage{makecell}
\usepackage{placeins}
\usepackage[switch]{lineno}

\newtheorem{theorem}{Theorem}

\title{Lookahead Branching for Neural Network Verification}

\author{
    Liam Davis$^1$ \and
    Duo Zhou$^2$ \and
    Huan Zhang$^2$ \and
    Guy Katz$^3$ \and
    Clark Barrett$^4$ \and
    Haoze Wu$^1$
    \affiliations
    $^1$Amherst College \\
    $^2$University of Illinois Urbana-Champaign \\
    $^3$Hebrew University of Jerusalem \\
    $^4$Stanford University
    \emails
    ljdavis27@amherst.edu,
    duozhou2@illinois.edu,
    huan@huan-zhang.com,
    g.katz@mail.huji.ac.il,
    barrettc@stanford.edu,
    hwu@amherst.edu
}

\begin{document}

\maketitle

\begin{abstract}
In this work, we investigate the effect of lookahead branching strategies in neural network verification. We present a general recipe to integrate lookahead into any branch-and-bound verifier and demonstrate how one of the current state-of-the-art branching heuristics, FSB, can be viewed as a special instantiation of the lookahead branching strategy. We also describe how, in addition to improving the quality of branching decisions, lookahead can generate additional lemmas that accelerate verification. We instantiate the method in two representative branch-and-bound-based verifiers (\marabou and \abcrown), and demonstrate that lookahead leads to consistent speedups in verification time and up to $57\%$ more solved instances. Code is available at \url{https://github.com/ai-ar-research/lookahead-branching}.
\end{abstract}

\section{Introduction}

Deep neural networks (DNNs) have become state-of-the-art solutions in various domains~\cite{sallab2017deep,mnih2013playing,he2015delving}. To ensure the deployment of DNN-based systems in safety critical domains, there has been significant effort from the formal methods and machine learning communities on developing scalable formal verification techniques that can reason about the behaviors of a neural network~\cite{katz2017reluplex,singh2019abstract,zhang2018efficient,wang2021beta}. State-of-the-art complete neural network verifiers are based on branch-and-bound (\bab), which involves performing case splitting on non-linear activation functions (e.g., ReLU) and analyzing the cases using an incomplete verifier; if the analysis is inconclusive, further case splits are recursively performed. 

The branching heuristic, the strategy to choose which case to split on, has a significant impact on verification efficiency. Ideally, the branching heuristic should lead to easier subproblems. Failing to do so, especially at the top of the search tree, can result in duplicated verification efforts and increased verification time. A number of branching heuristics have been developed for neural network verification~\cite{katz2017reluplex,wu2020parallelization,bunel2020branch,de2021scaling,wu2022efficient}. Most heuristics aim to make a decision \emph{quickly} by leveraging local information (e.g., variable bounds) gathered during the solving process. This increases the risk of ineffective branching, as decisions are made without evaluating the long-term impact of splits. Recent heuristics have shown it is beneficial to spend more effort making a branching decision by simulating branching on candidate neurons and evaluating its effect~\cite{de2021scaling}. This approach has culminated in Filtered Smart Branching (FSB), the standard branching heuristic in recent work.

Motivated by the success of FSB, we systematically study a general branching strategy that spends \emph{significant} effort making branching decisions. We adopt the terminology in formal methods and call this approach \emph{lookahead}~\cite{heule2009lookahead}. Lookahead involves simulating potential branching decisions by performing multiple splits to measure downstream effects. By analyzing the impact of each simulated branch, lookahead uses more information to make branching decisions. We present lookahead as a template algorithm with several tunable parameters such as the preselect strategy, lookahead depth, and evaluation metric that can be instantiated in a solver-specific manner. We discuss the choices for these parameters and how FSB can be viewed as a special instantiation of the lookahead procedure. In addition to informing the branching decision, lookahead might also discover new information, such as tightened variable bounds. We show this information can be used to derive valid lemmas and potentially prune the search space. We instantiate lookahead in two distinct \bab-verifiers, \marabou~\cite{katz2019marabou,wu2024marabou} and \abcrown~\cite{wang2021beta,xu2020automatic,zhang2022general,zhou2024scalable}, and demonstrate that lookahead can improve both tools. 

An overview of the lookahead workflow is presented in Figure~\ref{fig:overview}. At a given search state when branching is required, a set of split candidates is pre-selected and lookahead simulates branching on each candidate up to a certain depth. The outcome of each simulation is a score along with a number of tightened bounds, which might entail that certain unstable neurons are fixed to particular activation phases. In the end, lookahead outputs the next neuron to split on and a set of lemmas discovered during the lookahead simulations.

\begin{figure*}[t]
    \centering
    \includegraphics[width=0.9\textwidth]{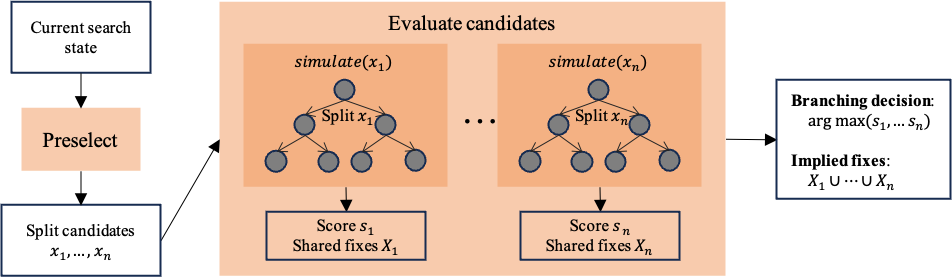}
    \caption{Visual overview of the lookahead procedure}
    \label{fig:overview}
\end{figure*}

The rest of the paper is organized as follows: Section~\ref{sec:related} reviews related work. Section~\ref{sec:prelim} provides background on neural network verification. Section~\ref{sec:method} presents the general lookahead approach and concrete instantiations in \marabou and \abcrown. Section~\ref{sec:eval} provides an evaluation of lookahead branching, comparing it to existing heuristics in \marabou and \abcrown. Finally, we conclude and discuss future directions in Section~\ref{sec:conclusion}.  

\section{Related Work}\label{sec:related}

Early efforts for complete verification of neural networks employed SMT and MILP solvers that enumerate activation patterns~\cite{katz2017reluplex}. A unified \bab view of verification was presented by ~\cite{bunel2020branch}. State-of-the-art complete verification tools including the SMT-based solvers~\cite{katz2019marabou,wu2024marabou} and GPU-accelerated bound-propagation-based tools~\cite{wang2021beta,xu2020automatic,zhang2022general,zhou2024scalable,shi2024genbab} can be viewed as instantiations of \bab. Branching heuristics have a significant impact on the runtime of complete verification tools. \textsc{BaBSR} introduced a ``strong-branching'' style score that solves a cheap surrogate relaxation for each candidate and became the default in many verifiers~\cite{bunel2020branch}. \textsc{FSB} (Filtered Smart Branching) refines this idea by first filtering candidates with BaBSR and then re-running a tighter bound computation on the shortlist~\cite{de2021scaling}. There has also been work on using Graph Neural Networks to train a policy for selecting splitting neurons~\cite{lu2019neural}, but such approaches require offline training and are typically tied to the distribution of training data, making them less directly comparable to general-purpose heuristics like ours. Our framework is complementary with learning-based heuristics, which could in principle be used as a preselect strategy within the lookahead template. Our work is inspired by \emph{lookahead} search in SAT and SMT solvers~\cite{heule2009lookahead} as well as MILP solvers~\cite{Glankwamdee2006Lookahead}, and we extend similar ideas to neural network verification.

\section{Background}\label{sec:prelim}
\subsection{Neural Network Verification}
For a trained deep neural network $N: \mathbb{R}^n \rightarrow \mathbb{R}^m$ with an input $x \in \mathbb{R}^n$ and output $y = N(x) \in \mathbb{R}^m$, the general DNN verification problem is whether or not there exists an input $x$ that produces an output $y$ that satisfies a property $\phi(y)$. If there exists an input $x$ that leads to an output $y$ that satisfies the property $\phi(y)$, then the problem is satisfiable (SAT); otherwise it is unsatisfiable (UNSAT).

\subsection{Bound Propagation}
To refine the search space of a neural network verification problem, bound propagation~\cite{singh2019abstract} estimates activation ranges at each layer, defining upper and lower bounds for each neuron. The Rectified Linear Unit (ReLU) activation function is defined as $\text{ReLU}(x) = \max(0, x)$. A ReLU neuron is considered active if its lower bound is strictly greater than zero ($\underline{z}^{(l)} > 0$), meaning its output will always be its input. Conversely, a ReLU neuron is inactive if its upper bound is less than or equal to zero ($\overline{z}^{(l)} \le 0$), meaning its output will always be zero. Otherwise, the activation status is not yet known and the ReLU in unstable. If a ReLU's bounds can guarantee the neuron is always active or inactive, unnecessary computations can be eliminated. Through this iterative process, the bounds on neuron activations are progressively tightened, leading the solver closer to a solution.

\subsection{Branch-and-bound}

\begin{figure}[t]
    \centering
    \includegraphics[width=0.8\columnwidth]{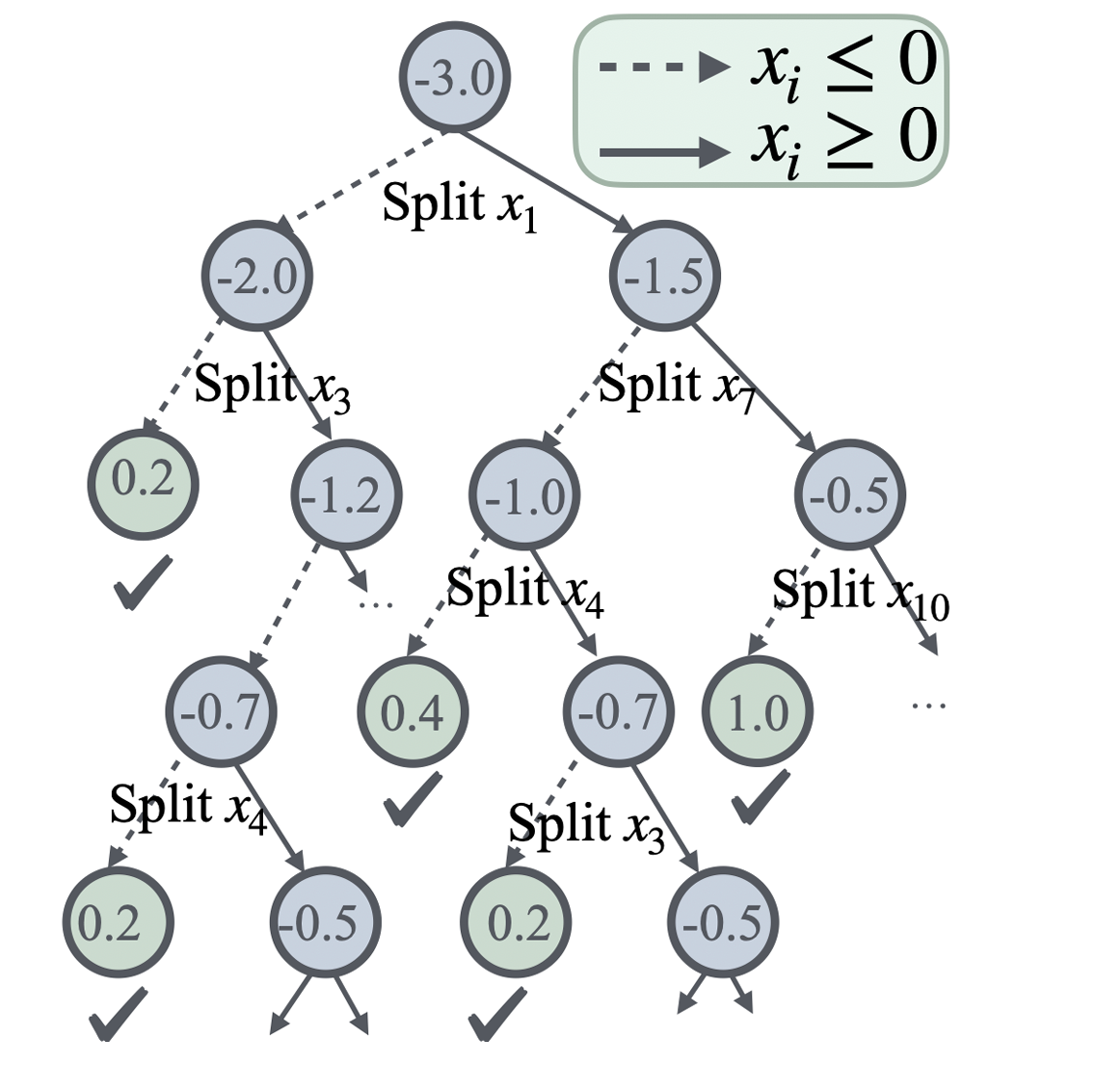}
    \caption{Each node encodes a BaB subproblem by splitting unstable ReLUs. Green nodes represent pruned nodes.~\protect\cite{zhou2024scalable}.}
    \label{fig:bab}
\end{figure}

The branch-and-bound (\bab) framework, illustrated by Figure \ref{fig:bab}, is an efficient approach to neural network verification. \bab systematically tightens the bounds of a neural network by splitting on unstable ReLU neurons. When an unstable ReLU is split, the problem is turned into two problems, one where the ReLU is active and one where the ReLU is inactive. To verify with \bab, ReLU splits are repeatedly applied until each subproblem can be definitively classified as either satisfying the property (SAT) or not satisfying it (UNSAT). 

The choice of ReLU to split on is imperative for the efficiency of \bab, as it determines how fast a solver converges to a solution; well-selected splits make significantly more progress to a solution than poorly-chosen splits. Thus, branching heuristics are essential to efficient verification.

\subsection{Branching Heuristics}
Several branching heuristics have been developed for \bab-based neural network verification. \textsc{BaBSR} is a deterministic heuristic that computes a surrogate relaxation of each neuron with the upper and lower bounds, and the biases of previous layers~\shortcite{bunel2020branch}. The unstable neuron with the highest score is then branched on. \textsc{Pseudo-impact branching} is a branching heuristic specific to \marabou. It estimates the impact a ReLU has on the Sum-of-Infeasibilities (SoI) of the network, calculated during \marabou's DeepSoI procedure~\cite{wu2022efficient}. The ReLU with the greatest estimated impact on the SoI is branched on. Pseudo-impact is presently the state-of-the-art heuristic implemented in \marabou. \textsc{FSB} pre-selects candidate ReLUs with BaBSR scores, and simulates one level of branching on each candidate before making a branching decision, and can be seen as a special instantiation of lookahead branching. FSB is presently the state-of-the-art heuristic implemented in \abcrown. We will discuss how lookahead extends FSB in section~\ref{sec:method}.

\section{Methodology}\label{sec:method}
In this section, we first present lookahead as a template algorithm and discuss its properties. We then discuss concrete strategies for instantiating the lookahead procedure. The algorithm begins by identifying a set of unstable neurons to simulate splits on. Given that performing a lookahead simulation on every unstable ReLU is computationally expensive, a preselect strategy is employed to shrink the candidate set. The preselect strategy differs based on implementation as detailed in Section~\ref{subsec:preselect}.

\begin{algorithm}[!t]
\caption{Lookahead Branching}
\label{alg:lookahead}
\begin{algorithmic}[1]
\STATE {\bfseries Input:} Set of unstable neurons $\mathcal{N}$ at the current search level
\STATE {\bfseries Output:} $\langle n^*, \mathcal{L}, \mathcal{U} \rangle$ where $n^*$ is the neuron to split, $\mathcal{L}$ and $\mathcal{U}$ are sound lower and upper bounds
\STATE {\bfseries Parameters:} lookahead depth $k$, preselect strategy $P$, heuristic $B$, scoring function $E$, aggregate strategy $C$
\STATE $\mathcal{N}' \gets P(\mathcal{N})$; $agScores \gets \{\}$; $\mathcal{L}, \mathcal{U} \gets [\,], [\,]$
\FOR{\textbf{each} $n \in \mathcal{N}'$}
    \STATE $scores', \ell, u \gets \text{CaseSplit}(n, k, \mathcal{N} \setminus \{n\})$
    \STATE $agScores[n] \gets C(scores')$; $\mathcal{L}.\text{append}(\ell)$; $\mathcal{U}.\text{append}(u)$
\ENDFOR
\STATE \textbf{return} $\arg \max (agScores),$
\STATE \hspace{2.15em} $\text{elementwise-max}(\mathcal{L}), \text{elementwise-min}(\mathcal{U})$
\STATE {\bfseries Procedure} \text{CaseSplit}($n, d, \mathcal{N}$)
\IF{$d = 0$}
    \STATE $score, \ell, u \gets E()$; \textbf{return} $[score], \ell, u$
\ELSE
    \STATE $\text{storeSolverState}()$; $scores, L, U \gets [\,], [\,], [\,]$
    \FOR{\textbf{each} $p \in \text{phases}(n)$}
        \STATE $\text{applySplit}(p)$; $\text{propagateBounds}()$
        \IF{$d = 1$} \STATE $n' \gets B(\mathcal{N})$ \ELSE \STATE $n' \gets nil$ \ENDIF
        \STATE $scores', \ell, u \gets \text{CaseSplit}(n', d - 1, \mathcal{N} \setminus \{n'\})$
        \STATE $scores \gets scores \mathbin{::} scores'$; $L.\text{append}(\ell)$; $U.\text{append}(u)$; $\text{restoreSolverState}()$
    \ENDFOR
    \STATE \textbf{return} $scores,$
    \STATE \hspace{2.15em} $\text{elementwise-min}(L), \text{elementwise-max}(U)$
\ENDIF
\end{algorithmic}
\end{algorithm}

For each preselected candidate, the \casesplit sub-routine simulates a case split on the neuron, creating one branch for each activation phase. For the ReLU activation function, there would be two phases (active and inactive). The algorithm then recursively explores the next level of the search tree by selecting another neuron to split on, using the \baseSelect heuristic. This process continues until a specified lookahead depth (\paramDepth) is reached. The goal of performing additional splits is to explore the cascading effects of each candidate split. After each split is simulated, the solver performs bound propagation. When the lookahead depth is reached, the \eval function is called to evaluate the current solver state and collect the current variable bounds. Importantly, the solver state is saved and restored after each simulated split to ensure that the lookahead simulation does not interfere with the actual search process. In the end, the \casesplit function returns the score of each simulated leaf, as well as the loosest lower and upper bounds of the variables discovered during the simulation.

The scores from all simulated branches are aggregated using the specified aggregation strategy to produce a single score for each candidate neuron. The neuron with the best score is selected as the next split. Moreover, the tightest lower and upper bounds discovered during lookahead are returned, which can be used to further prune the search space.

In general, lookahead can be computationally expensive, as it requires repeated simulation of the solving process. In practice, it is beneficial to invoke Algorithm~\ref{alg:lookahead} at the top of the search tree, before falling back to more efficient heuristics. The key advantage of lookahead is that it considers the cascading effects of a branching decision. A neuron that appears promising based on local information may lead to poor downstream progress, while a seemingly less promising candidate can yield improvements over several branching decisions.

FSB can be viewed as a particular implementation of the lookahead procedure where BaBSR is used as the \baseSelect heuristic, \paramDepth is 1, and the strategy is applied uniformly at every branching decision. This design prioritizes efficiency by minimizing per-decision overhead. But by simulating only one branching decision, FSB cannot fully capture the cascading effects that propagate through multiple levels of the search tree. Our work explores whether investing additional computation in critical branching decisions can improve verification performance. Specifically, we investigate depth-2 lookahead applied at the top of the search tree, where branching decisions have the largest downstream impact. This design is motivated by the strong branching literature in MILP~\cite{Glankwamdee2006Lookahead}, which demonstrates that expensive branching heuristics provide the greatest benefit early in search. Additionally, we explore how phase-fixing lemmas discovered during lookahead (Section~\ref{subsec:phase_fixing}) can further prune the search space. Section~\ref{sec:eval} evaluates whether these design choices yield net performance improvements despite their computational overhead.

\subsection{Preselect Strategies}
\label{subsec:preselect}
If lookahead were to be performed on every neuron in a large neural network, the computational overhead would outweigh the improvement in branch quality. Thus, it is essential to preselect a subset of neurons on which to run lookahead. Several preselect strategies are possible.

One approach is to use a polarity-based strategy, that uses a heuristic that scores neurons based on the sensitivity of their activation status to changes in their bounds. For example, a score such as $\max\left(\frac{ub}{lb}, \frac{lb}{ub}\right)$, where $ub$ and $lb$ are the upper and lower bounds of a neuron, can be used to identify neurons whose activation status is most undecided. Neurons with the highest scores are then selected for lookahead. We used this polarity-based preselect strategy in \marabou. Another approach is to use existing branching scores, such as BaBSR scores, to rank neurons. In this case, a fixed number of neurons with the highest scores are chosen for lookahead. This preselect strategy was used in \abcrown. Regardless of the strategy, the preselect method should identify promising candidates using lightweight metrics.

\subsection{Scoring Functions}
\label{subsec:scoring}

The scoring function is a central component of lookahead branching, as it quantifies the effectiveness of each simulated split and guides the branching decision. Two general classes of scoring metrics are commonly used: neuron-fixing-based metrics and bound-reduction-based metrics.

\paragraph{Neuron-fixing-based metric.}
This approach evaluates a split by counting how many previously unstable neurons become phase-fixed (i.e., their activation status is determined) after bound propagation in each branch. To encourage both high total progress and balanced outcomes, a balance score is computed for each split. For example, if a split results in $a$ neurons fixed in one branch and $b$ in the other, the score can be defined as $\frac{a \times b}{a + b + 1}$. This formula favors splits that both fix many neurons and distribute the fixes evenly between branches. Consider a split that fixes 9 neurons in one branch and 1 in the other. Its score would be $\frac{9 \times 1}{9 + 1 + 1} \approx 0.8$. In contrast, a split that fixes 5 neurons in each branch would yield a score of $\frac{5 \times 5}{5 + 5 + 1} \approx 2.3$. Our metric would then favor the more balanced outcome in case the same number of neurons are fixed. When lookahead is performed to greater depths, the scores are computed recursively: at the leaves, the score is based on the number of phase fixes, and at each internal node, the balance formula is applied to the scores of its child branches, propagating upward to yield a final score. This polarity-based metric was used in \marabou.

\paragraph{Bound-reduction-based metric.}
This metric is based on the reduction in variable bounds or other continuous measures of progress, such as the width of neuron bounds or their proximity to a decision threshold. For example, a scoring function may combine the width of a neuron's bounds, its distance from zero, and an estimate of its influence on the objective (e.g., via gradient approximation). In a lookahead setting, after simulating a split, the scoring function can aggregate the scores from subsequent branches using a balance formula similar to the neuron-fixing metric. For instance, if $S^+$ and $S^-$ are the scores from the two branches after a split, the lookahead score can be defined as $S_0 + \lambda \cdot \frac{S^+ \times S^-}{S^+ + S^- + 1}$, where $S_0$ is the immediate score for the split and $\lambda$ is a discount factor to control the influence of deeper lookahead. For deeper lookahead, this aggregation is applied recursively as scores are propagated up the lookahead tree. This bound-reduction-based metric was used in \abcrown.

Regardless of the specific metric, an effective scoring function should capture both the potential for maximal progress toward a solution and the balance of outcomes across different branches. This ensures that the branching decision not only accelerates convergence but also avoids highly imbalanced splits that may lead to inefficient search.

\subsection{Phase Fixing via Lookahead Splits}
\label{subsec:phase_fixing}

One significant outcome of the lookahead branching procedure is its ability to refine variable bounds, which can lead to phase fixing of previously unstable neurons. Specifically, during the lookahead process, bound propagation is applied after simulating splits, and the resulting bounds can sometimes determine that a neuron is stable (i.e., its phase is fixed). The following theorem formalizes this effect:

\begin{theorem}[Phase Fixing via Lookahead]
Let $z = \mathrm{ReLU}(y)$ be an unstable ReLU, i.e., $\ell_y < 0 < u_y$. Suppose we simulate a split on another unstable ReLU $z_r = \mathrm{ReLU}(y_r)$, generating two subproblems corresponding to the inactive ($y_r \leq 0$) and active ($y_r \geq 0$) cases.

Let $\ell_y^{\mathrm{inact}}$ and $\ell_y^{\mathrm{act}}$ be the lower bounds on $y$ obtained in each subproblem. Then the refined lower bound $\ell_y^{\mathrm{new}} := \min(\ell_y^{\mathrm{inact}}, \ell_y^{\mathrm{act}})$ satisfies $\ell_y^{\mathrm{new}} \geq \ell_y$. Similarly, the refined upper bound $u_y^{\mathrm{new}} := \max(u_y^{\mathrm{inact}}, u_y^{\mathrm{act}})$ satisfies $u_y^{\mathrm{new}} \leq u_y$.

As a consequence: if $\ell_y^{\mathrm{new}} \geq 0$, the phase of $z$ is fixed to \emph{active}; if $u_y^{\mathrm{new}} \leq 0$, it is fixed to \emph{inactive}.
\end{theorem}

\begin{proof}[Proof Sketch]
Simulating a split on $z_r$ refines the input domain into two disjoint subdomains. Bound propagation on each subproblem yields tighter constraints on all dependent variables, including $y$. Because the union of the subdomains recovers the full feasible region, the maximum lower bound and minimum upper bound across the subproblems provide valid global bounds.
\end{proof}

The phase-fixing capability cannot be used when either CROWN or \acrown bound propagation is used, as the adaptive ReLU rule means bounds can become looser when pre-activation bounds are tightened. Therefore we used phase-fixing in \marabou but not \abcrown.

\section{Evaluation}\label{sec:eval}
To evaluate the effectiveness of lookahead branching, we implemented Algorithm~\ref{alg:lookahead} in two state-of-the-art \bab-based verification tools \marabou and \abcrown. The two verifiers employ different paradigms: \marabou is a CPU-based verifier that employs an SMT-based incomplete decision procedure, while \abcrown runs GPU-accelerated bound-propagation. We evaluate lookahead branching against the state-of-the-art heuristics in each verifier on various benchmarks, investigating whether lookahead branching can improve the branch-and-bound across distinct solver paradigms.

\subsection{Case Study on \marabou}

\subsubsection{Experimental Setup}
For experimentation on Marabou, we compared lookahead branching to BaBSR and pseudo-impact, two of Marabou's existing heuristics. BaBSR is a widely-used deterministic heuristic~\cite{bunel2020branch}, while pseudo-impact is a dynamic heuristic specific to \marabou~\cite{wu2022efficient}. We conducted experiments on three different benchmark sets, NAP, NN4Sys, and MNIST. NAP is a benchmark designed to evaluate neural network verifiers on specifications defined by neural activation patterns, which characterize broad, numerically challenging regions of the input space. NN4Sys is a benchmark suite constructed from neural networks used in computer systems, testing real-world verification challenges that have been proven difficult in recent iterations of the VNN Competition. The MNIST dataset includes various feed forward neural networks for handwritten digit recognition. The MNIST network architectures we tested include, in layers by neurons, 20x20, 2x256, 4x256, and 6x256. We implemented lookahead on top of the most recent version of \marabou.

We use the polarity-based pre-selecting strategy as described in Section~\ref{subsec:preselect} to select 10 candidate neurons. We use a \paramDepth of 2, and instantiate \baseSelect with two different heuristics, BaBSR and pseudo-impact. In the experiments, we perform lookahead splits on the top five branching levels, and fall back to the base heuristics for the rest of it. We used the neuron-fixed-based metric for evaluating a search state and combining the scores across subproblems~\ref{subsec:scoring}. The experiments were run on a server with 2.6-GHz AMD CPUs, with 4 CPU cores allocated per benchmark. Each benchmark was given a 30 minute time limit and a 32 GB memory limit. 

\subsubsection{Results}

Table \ref{tab:marabou} presents a comparison of the performance of Marabou with and without lookahead branching across the various benchmarks. The tables report the number of instances solved within the time limit (1800 seconds), along with average run time on solved instances. Overall, lookahead branching results in a substantial increase in the number of solved instances for both BaBSR and pseudo-impact heuristics.

\begin{table}[t]
    \centering
    \setlength{\tabcolsep}{1.8pt} 
    \renewcommand{\arraystretch}{1.2}
    \resizebox{\columnwidth}{!}{%
        \footnotesize 
        \begin{tabular}{l c cc cc cc cc}
            \toprule
            \textbf{Benchmark} & \textbf{\#} & \multicolumn{2}{c}{\textbf{babsr}} & \multicolumn{2}{c}{\textbf{babsr+lh}} & \multicolumn{2}{c}{\textbf{p-i}} & \multicolumn{2}{c}{\textbf{p-i+lh}} \\
            \cmidrule(lr){3-4} \cmidrule(lr){5-6} \cmidrule(lr){7-8} \cmidrule(lr){9-10}
             & \makecell{\textbf{Bench.}} & \makecell{Sol.} & \makecell{Time} & \makecell{Sol.} & \makecell{Time} & \makecell{Sol.} & \makecell{Time} & \makecell{Sol.} & \makecell{Time} \\
            \midrule
            NAP          & 235 & 30  & 29.0 & \textbf{47}  & 3.7  & 19  & 0.5  & \textbf{24} & 3.0 \\
            NN4SYS       & 120 & \textbf{80}  & 127.1 & \textbf{80}  & 114.9 & 82  & 149.9 & \textbf{94} & 238.1 \\
            MNIST 20x20  & 500 & 183 & 100.5 & \textbf{192} & 130.6 & 130 & 14.9 & \textbf{141} & 9.3 \\
            MNIST 2x256  & 500 & \textbf{370} & 28.3 & 369 & 30.3 & 405 & 77.9 & \textbf{407} & 83.1 \\
            MNIST 4x256  & 500 & 281 & 30.1 & \textbf{286} & 30.5 & \textbf{298} & 72.5 & 295 & 52.2 \\
            MNIST 6x256  & 500 & 248 & 86.1 & \textbf{254} & 94.0 & 240 & 42.5 & \textbf{240} & 40.9 \\
            \bottomrule
        \end{tabular}%
    }
    \caption{Marabou results on various benchmarks. P-I denotes pseudo-impact, LH denotes lookahead. Time is in seconds.}
    \label{tab:marabou}
\end{table}

\begin{figure*}[!t]
    \centering
    \begin{subfigure}{0.3\textwidth}
        \centering
        \includegraphics[width=\linewidth]{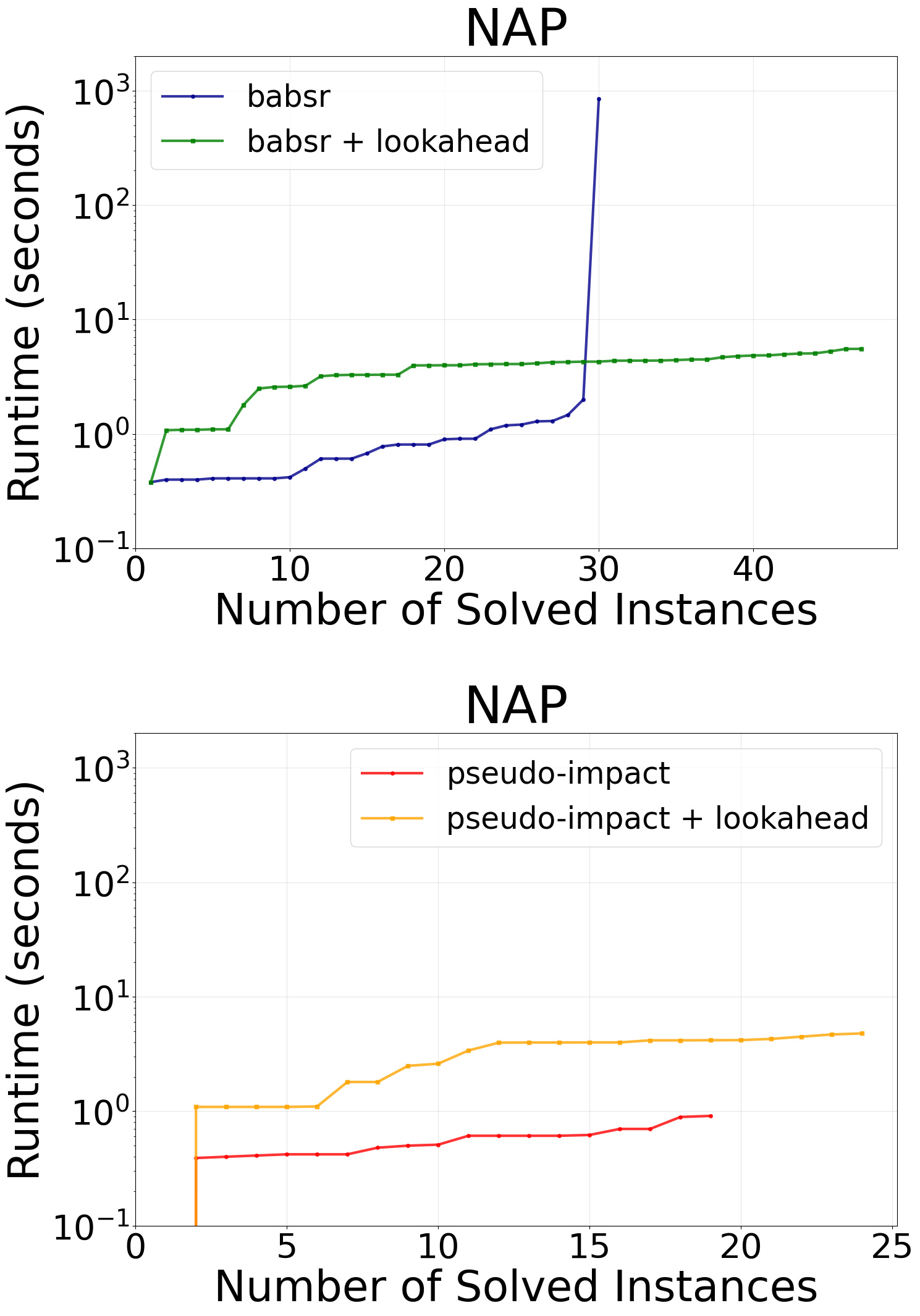}
        \caption{NAP}
        \label{fig:marabou_cactus_nap}
    \end{subfigure}
    \hfill
    \begin{subfigure}{0.3\textwidth}
        \centering
        \includegraphics[width=\linewidth]{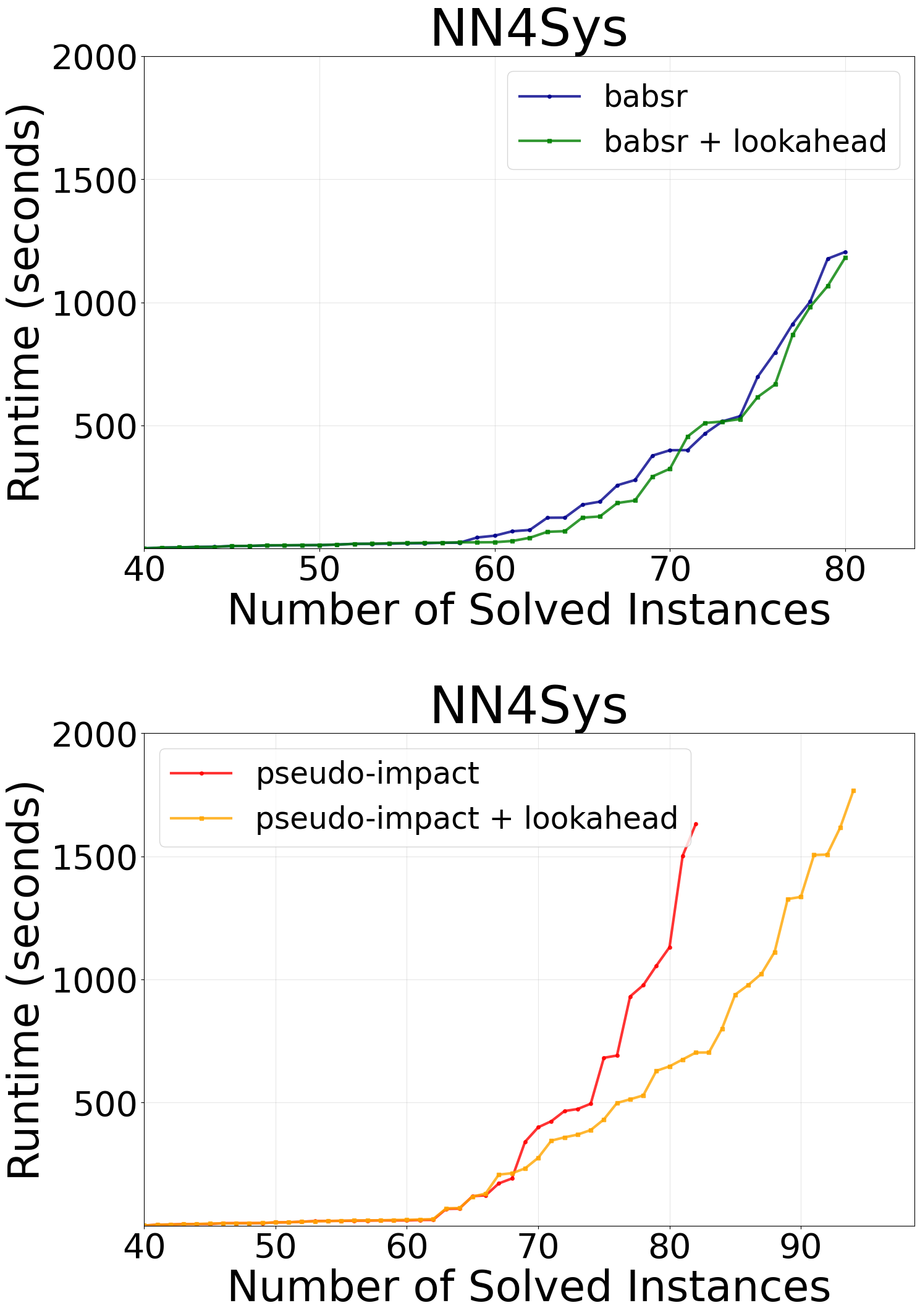}
        \caption{NN4Sys}
        \label{fig:marabou_cactus_nn4sys}
    \end{subfigure}
    \hfill
    \begin{subfigure}{0.3\textwidth}
        \centering
        \includegraphics[width=\linewidth]{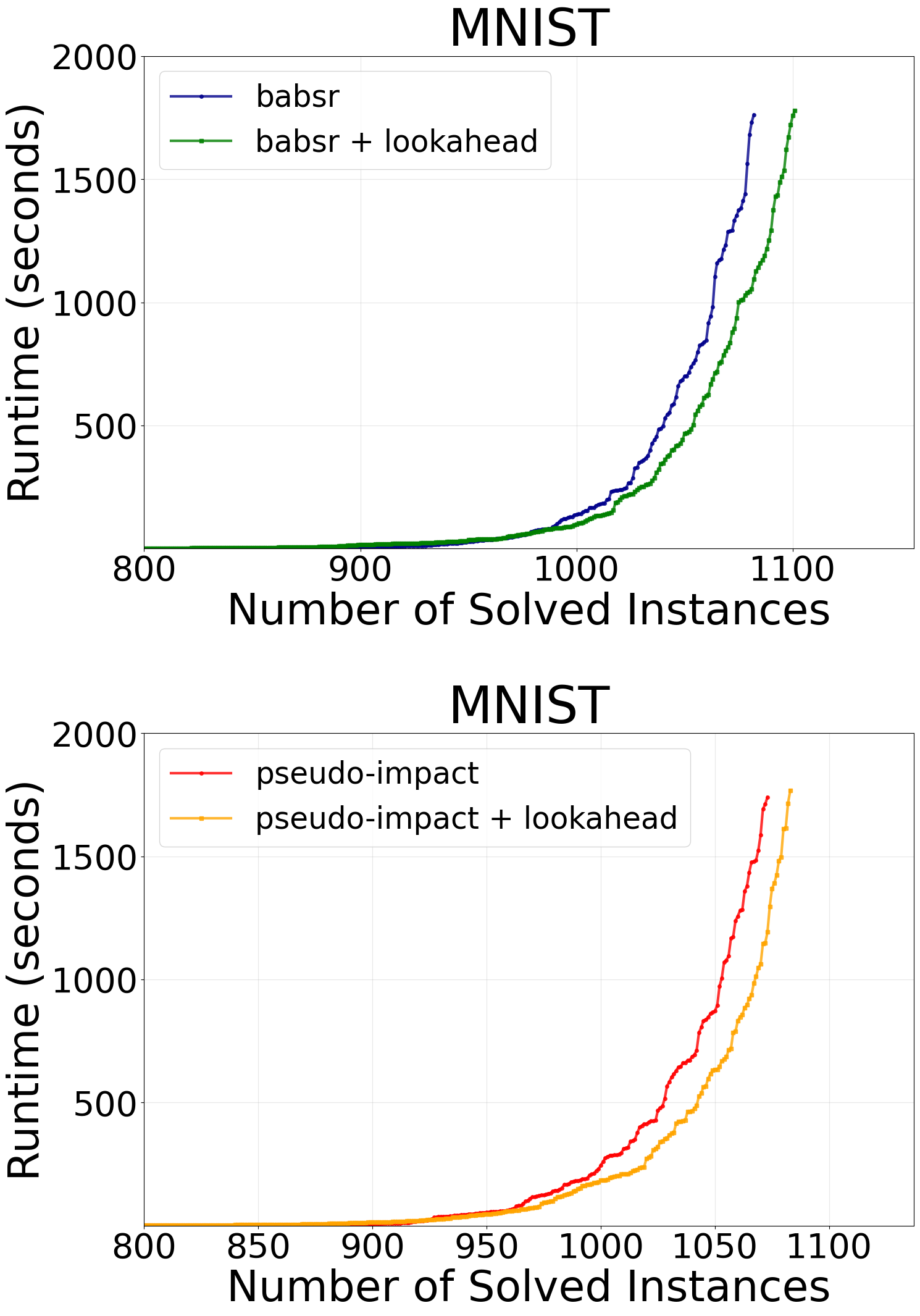}
        \caption{MNIST}
        \label{fig:marabou_cactus_mnist}
    \end{subfigure}
    \caption{Cactus plots comparing existing heuristics and lookahead on \marabou.}
    \label{fig:marabou_cactus}
\end{figure*}

Figure~\ref{fig:marabou_cactus} provides cactus plots comparing lookahead against BaBSR and pseudo-impact on the three benchmark sets. Overall, lookahead leads to more solved instances compared to the baseline heuristics, especially at higher time limits. On NAP, overhead makes lookahead slower at low time limits, but branching quality gains dominate overall. On NN4Sys, lookahead yields significantly more solved instances with pseudo-impact and solve time improvements with BaBSR. For MNIST, lookahead solves more instances at higher time limits. Table~\ref{tab:lookahead-overhead} compares lookahead's overhead an easy and hard instance. On the easy instance, the overhead dominates. On the hard instance, the overhead constitutes a small part of overall runtime and the improved branching yields a substantial speedup.

\begin{table}[h]
\centering
\setlength{\tabcolsep}{3pt}
\renewcommand{\arraystretch}{1.2}
\resizebox{\columnwidth}{!}{%
    \footnotesize
    \begin{tabular}{lcccc}
    \toprule
    \textbf{Instance} & \textbf{Baseline} & \textbf{LH} & \textbf{LH Phase (s)} & \textbf{Overhead} \\
    \midrule
    \textsc{NAP cls0\_id41}      & 0.22   & 2.48  & 2.17 & 87.5\% \\
    \textsc{NN4Sys lindex\_1\_5} & 537.62 & 66.14 & 7.01 & 10.6\% \\
    \bottomrule
    \end{tabular}%
}
\caption{Lookahead overhead on easy vs.\ hard instances, time in seconds.}
\label{tab:lookahead-overhead}
\end{table}

\subsubsection{Ablation Studies on \marabou}\label{sec:marabou_ablation}

We performed additional ablation studies on the NAP and NN4Sys benchmarks by varying the lookahead depth (D), number of preselected candidates (C), and removing phase fixing (PF). All configurations use the polarity-based preselect strategy and BaBSR as the base select heuristic.

\begin{table}[t]
    \centering
    \setlength{\tabcolsep}{7pt}
    \resizebox{\columnwidth}{!}{%
        \small
        \begin{tabular}{@{}l rr rr@{}}
            \toprule
            & \multicolumn{2}{c}{\textbf{NAP}} & \multicolumn{2}{c}{\textbf{NN4Sys}} \\
            \cmidrule(lr){2-3}\cmidrule(lr){4-5}
            \textbf{Configuration} & Solved & Time (s) & Solved & Time (s) \\
            \midrule
            LH D5C10         & \textbf{47} & 3.7  & 80 & 114.9 \\
            LH D1C10         & 36 & 1.5   & 80 & 118.8 \\
            LH D4C10         & 36 & 1.6   & 80 & 121.97 \\
            LH D6C10         & 36 & 1.59  & 80 & 124.73 \\
            LH D8C10         & 36 & 1.48  & 80 & 121.93 \\
            LH D10C10        & 45 & 5.1   & 79 & 120.1 \\
            LH D5C5          & \textbf{47} & 2.1  & 80 & 118.5 \\
            LH D5C20         & 39 & 6.6   & 80 & 121.9 \\
            LH D5C10 no PF   & \textbf{47} & 3.3  & 80 & 147.6 \\
            \bottomrule
        \end{tabular}%
    }
    \caption{Ablation studies on \marabou with BaBSR. D denotes lookahead depth, C denotes preselected candidates, PF denotes phase fixing.}
    \label{tab:marabou_ablation}
\end{table}

On NAP, all lookahead configurations solve between 36 and 47 instances, compared to 30 for BaBSR. On NN4Sys, varying depth and candidate count has minimal effect on performance, with the exception of the D10C10 configuration, where increased branching overhead reduces performance. Comparing lookahead with and without phase fixing shows that phase fixing has no effect on NAP but yields a 33-second improvement on NN4Sys.

\subsection{Case Study on \abcrown}

\subsubsection{Experimental Setup}
For experimentation on \abcrown, we compared lookahead branching to its current state-of-the-art heuristic, FSB~\cite{de2021scaling}, on seven different convolutional neural networks ranging with various robustness properties for each. The networks came from the CIFAR, MNIST, and TinyImageNet datasets, three computer vision datasets for image and text recognition. The networks from the CIFAR and TinyImageNet datasets range between 14.4 million and 31.6 million parameters, making them challenging verification queries. The MNIST benchmarks were adversarially trained, introducing complexity that makes formal verification particularly challenging. We implemented lookahead in the most recent version of \abcrown ~\cite{zhou2025clipandverify}. 

With \abcrown, the solver first attempted to solve the problem with an adversarial attack algorithm~\cite{zhang22babattack}, then with incomplete verification using auto-LiRPA~\cite{xu2021fast} before beginning complete verification with \bab. When \bab is run, the first five \bab rounds are done using lookahead branching with a \paramDepth of 2, and then FSB branching is used. We use BaBSR as the preselect strategy and select 15 candidate neurons.  We used the bound-reduction-based metric for the scoring function (Section~\ref{subsec:scoring}) with a discount factor $\lambda$ of 0.5. We did not implement the phase fixing techniques in \abcrown, as \abcrown leverages \acrown for bound propagation and thus phase fixing is not sound. The experiments were run on a server with 2.6-GHz AMD CPUs and A100 GPUs. One A100 GPU and 96 CPU cores were allocated for each experiment, and a 128 GB memory limit was given for each experiment.

\subsubsection{Results}

\begin{table}[t]
    \centering
    \setlength{\tabcolsep}{1.5pt}
    \renewcommand{\arraystretch}{1.2}
    \resizebox{\columnwidth}{!}{%
        \footnotesize
        \begin{tabular}{ll c c cc cc}
            \toprule
            & & & \makecell{\textbf{T/O}} & \multicolumn{2}{c}{\textbf{Solved}} & \multicolumn{2}{c}{\textbf{Avg Time (s)}} \\
            \cmidrule(lr){5-6} \cmidrule(lr){7-8}
            \textbf{Dataset} & \textbf{Model} & \textbf{\#} & \textbf{(s)} & FSB & LH & FSB & LH \\
            \midrule
            CIFAR & CIFAR100 & 200 & 360 & 112 & 112 & 19.49 & \textbf{16.47} \\
                  & CIFAR10-ResNet & 72 & 360 & 59 & \textbf{60} & 22.83 & \textbf{21.87} \\
                  & CNN-A-Mix & 200 & 200 & 83 & 83 & 7.58 & \textbf{6.03} \\
                  & CNN-B-Adv & 200 & 450 & 93 & 93 & 10.23 & \textbf{8.60} \\
            \midrule
            MNIST & CNN-A-Adv & 200 & 200 & 141 & 141 & 11.34 & \textbf{8.90} \\
            \midrule
            TinyImageNet & tinyimagenet & 200 & 360 & 129 & \textbf{130} & 23.84 & \textbf{20.83} \\
            \bottomrule
        \end{tabular}%
    }
    \caption{\abcrown results on various benchmarks. T/O denotes timeout, LH denotes lookahead.}
    \label{tab:abcrown}
\end{table}

\begin{figure*}[!t]
    \centering
    \includegraphics[width=0.9\textwidth]{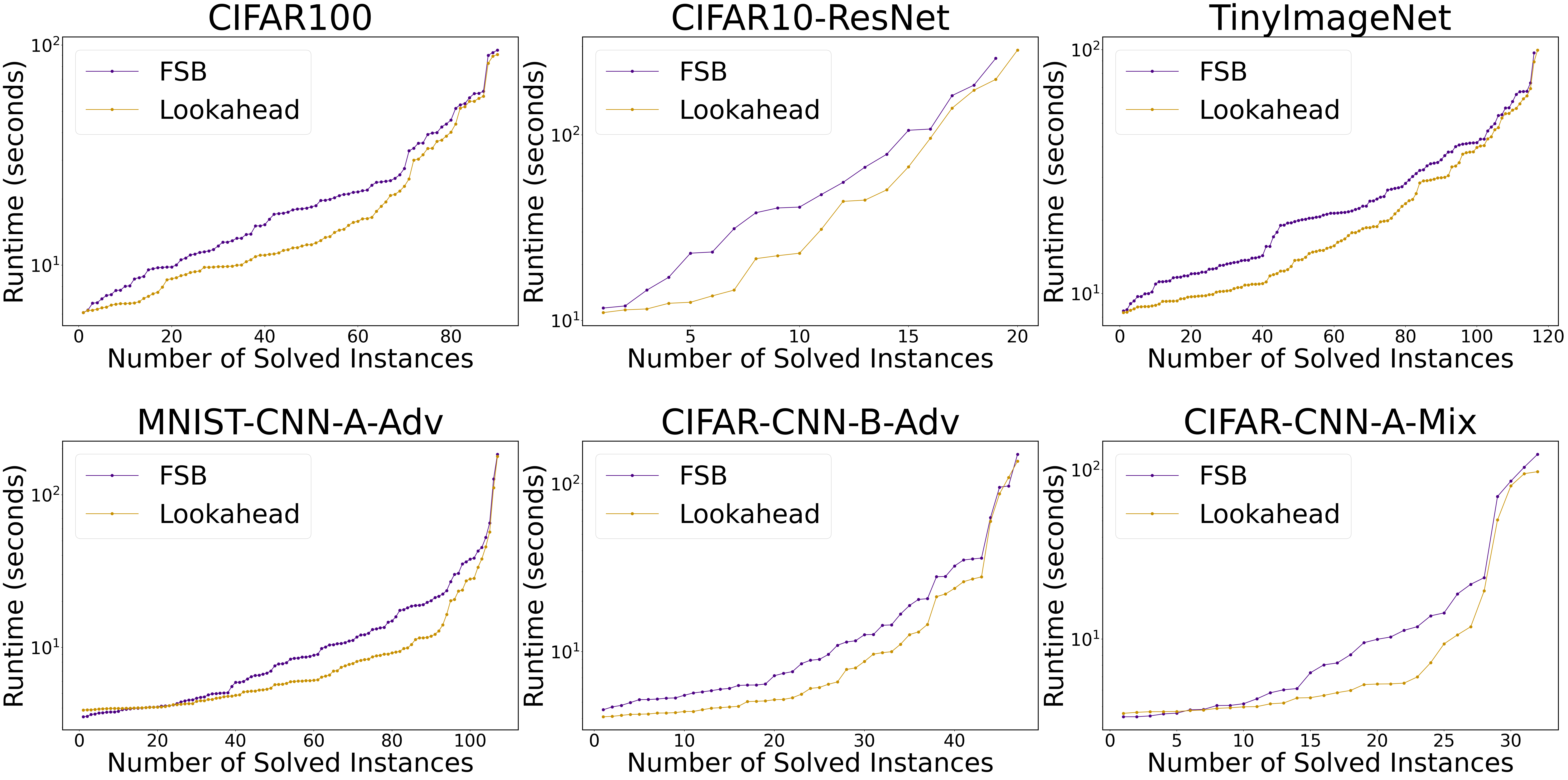}
    \caption{Cactus plots comparing FSB and lookahead on \abcrown.}
    \label{fig:betacrown_cactus}
\end{figure*}

Table \ref{tab:abcrown} presents a comprehensive overview of the verification performance on the seven neural network models using both FSB branching and lookahead branching within \abcrown. The table compares the total number of solved instances (including both SAT and UNSAT) and the average runtime on solved instances. To isolate the impact of lookahead branching on the branch-and-bound core of \abcrown, Table \ref{tab:abcrown_unsat} presents results filtered to include only instances that required the full \bab procedure. SAT instances solved through adversarial attacks and easier UNSAT instances solved with incomplete verification are excluded. We see that lookahead leads to a consistent speedup in solve time, and contributes two unique solutions.

\begin{table}[t]
    \centering
    \setlength{\tabcolsep}{2.5pt}
    \renewcommand{\arraystretch}{1.2}
    \resizebox{\columnwidth}{!}{%
        \footnotesize
        \begin{tabular}{ll cc cc c}
            \toprule
            & & \multicolumn{2}{c}{\textbf{Solved}} & \multicolumn{2}{c}{\textbf{Avg Time (s)}} & \makecell{\textbf{Avg}\\ \textbf{Speedup}} \\
            \cmidrule(lr){3-4} \cmidrule(lr){5-6}
            \textbf{Dataset} & \textbf{Model} & FSB & LH & FSB & LH & \textbf{(\%)} \\
            \midrule
            CIFAR & CIFAR100 & 90 & 90 & 23.34 & \textbf{19.58} & 16.1 \\
                  & CIFAR10-ResNet & 19 & \textbf{20} & 69.51 & \textbf{64.27} & 24.3 \\
                  & CNN-A-Mix & 32 & 32 & 19.10 & \textbf{15.08} & 15.4 \\
                  & CNN-B-Adv & 47 & 47 & 19.77 & \textbf{16.54} & 21.9 \\
            \midrule
            MNIST & CNN-A-Adv & 107 & 107 & 14.77 & \textbf{11.53} & 16.3 \\
            \midrule
            TinyImageNet & tinyimagenet & 116 & \textbf{117} & 26.11 & \textbf{22.74} & 21.0 \\
            \bottomrule
        \end{tabular}%
    }
    \caption{\abcrown results on UNSAT instances solved with \bab. LH denotes lookahead, timeouts remain the same as table~\ref{tab:abcrown}.}
    \label{tab:abcrown_unsat}
\end{table}

\subsubsection{Ablation Studies on \abcrown}
\label{sec:abcrown_ablation}

\begin{table}[t]
    \centering
    \setlength{\tabcolsep}{2pt}
    \renewcommand{\arraystretch}{1.2}
    \resizebox{\columnwidth}{!}{%
        \footnotesize
        \begin{tabular}{l cc cc cc}
            \toprule
            & \multicolumn{2}{c}{\textbf{1 LH Branch}} 
            & \multicolumn{2}{c}{\textbf{5 LH Branches}} 
            & \multicolumn{2}{c}{\textbf{10 LH Branches}} \\
            \cmidrule(lr){2-3} \cmidrule(lr){4-5} \cmidrule(lr){6-7}
            \textbf{Model} 
            & \makecell{Solved} & \makecell{Time} 
            & \makecell{Solved} & \makecell{Time} 
            & \makecell{Solved} & \makecell{Time} \\
            \midrule
            CIFAR100       & 90 & 64.21 & 90 & 64.27 & 90 & \textbf{57.25} \\
            CIFAR10-ResNet & 20 & 20.40 & 20 & \textbf{19.58} & 20 & 21.49 \\
            CNN-A-Mix      & 32 & 16.15 & 32 & 15.08 & 32 & \textbf{14.99} \\
            CNN-B-Adv      & 47 & 17.24 & 47 & \textbf{16.54} & 47 & 16.86 \\
            CNN-A-Adv      & 107 & 11.89 & 107 & \textbf{11.53} & 107 & 13.65 \\
            tinyimagenet   & 117 & 24.66 & 117 & \textbf{22.74} & 117 & 24.93\\
            \bottomrule
        \end{tabular}%
    }
    \caption{\abcrown results with varying number of lookahead branches. LH denotes lookahead. Time is in seconds.}
    \label{tab:abcrown_ablation}
\end{table}

Table \ref{tab:abcrown_ablation} presents the performance of \abcrown while varying the number of lookahead branches performed before switching to FSB. Across all hyperparameters tested, lookahead either solves more instances or has faster solve times than FSB. Changing the number of lookahead branches does not change the number of solved instances and changes the solve time minimally.

Overall, we found that lookahead can boost the performance of \bab in two fundamentally different neural network verifiers. This suggests that lookahead branching can be a generally applicable strategy for enhancing the performance of complete neural network verification.

\section{Conclusion}\label{sec:conclusion}
In this paper, we introduced a general lookahead branching strategy for neural network verification. The key idea is to simulate candidate splits to make more informed branching decisions. We discussed design choices of lookahead and instantiated lookahead branching for two distinct complete verifiers, \marabou and \abcrown. Our results show that using lookahead branching results in substantial performance gains in both solvers across a wide range of benchmarks. 

\paragraph{Limitations.} As we presented lookahead as a template algorithm, its design space is massive. While we considered two instantiations in \marabou and one in \abcrown, it would be interesting to evaluate more variants of lookahead. In particular, we plan to investigate deeper lookahead simulations or periodically invoking lookahead later in the search. In addition, while we explored leveraging lookahead to fix unstable neurons, it might be interesting to leverage other information, such as dependencies between neurons, which might result in additional pruning of the search space.

\newpage
\section*{Acknowledgments}

This work made use of high-performance computing equipment funded by the National Science Foundation under Grant \#2117377. Wu's work is supported by a gift from VMware University Research Fund. Additional support was provided by the Stanford Center for Automated Reasoning (Centaur). Huan Zhang acknowledges the support from the Schmidt Sciences AI2050 program and the National Science Foundation under Grant \#2331967.

%% The file named.bst is a bibliography style file for BibTeX 0.99c
\bibliographystyle{named}
\bibliography{ijcai26}

\end{document}